\documentclass[10pt,twocolumn,letterpaper]{article}

\usepackage{cvpr}              

\usepackage{graphicx}
\usepackage{amsmath}
\usepackage{amssymb}
\usepackage{booktabs}
\usepackage{comment}
\usepackage{mathtools} 
\usepackage{siunitx}
\usepackage{multicol}
\usepackage{multirow}

\usepackage[pagebackref,breaklinks,colorlinks]{hyperref}

\usepackage[capitalize]{cleveref}
\crefname{section}{Sec.}{Secs.}
\Crefname{section}{Section}{Sections}
\Crefname{table}{Table}{Tables}
\crefname{table}{Tab.}{Tabs.}

\usepackage{amsmath, amssymb, amsthm}
\newtheorem{proposition}{Proposition}

\def\cvprPaperID{5} 
\def\confName{CVPR}
\def\confYear{2026}

\begin{document}

\title{Towards Minimal Focal Stack in Shape from Focus}

\author{
Khurram Ashfaq \quad Muhammad Tariq Mahmood\thanks{Corresponding author}\\
Korea University of Technology and Education\\
{\tt\small khurram@koreatech.ac.kr \quad tariq@koreatech.ac.kr}
}
\maketitle

\begin{abstract}
Shape from Focus (SFF) is a depth reconstruction technique that estimates scene structure from focus variations observed across a focal stack, that is, a sequence of images captured at different focus settings. A key limitation of SFF methods is their reliance on densely sampled, large focal stacks, which limits their practical applicability. In this study, we propose a focal stack augmentation that enables SFF methods to estimate depth using a reduced stack of just two images, without sacrificing precision.  We introduce a simple yet effective physics-based focal stack augmentation that enriches the stack with two auxiliary cues: an all-in-focus (AiF) image estimated from two input images, and Energy-of-Difference (EOD) maps, computed as the energy of differences between the AiF and input images. Furthermore, we propose a deep network that computes a deep focus volume from the augmented focal stacks and iteratively refines depth using convolutional Gated Recurrent Units (ConvGRUs) at multiple scales. Extensive experiments on both synthetic and real-world datasets demonstrate that the proposed augmentation benefits existing state-of-the-art SFF models, enabling them to achieve comparable accuracy. The results also show that our approach maintains state-of-the-art performance with a minimal stack size.
\end{abstract}


\section{Introduction}

Shape-from-Focus (SFF) is a passive optical technique that recovers dense depth by analyzing a focal stack, a sequence of images captured at varying focus distances \cite{nayar1994shape}. The scene’s depth range is discretely sampled using an appropriate step size, which determines both the change in focus distance between successive captures and the number of images in the focal stack \cite{muhammad2012sampling, pertuz2015efficient}. The stack size is critical: too many images slow down the process, while too few degrade depth accuracy.

Conventional SFF methods estimate depth by computing a focus measure at discrete focus steps, making the number of images in the focal stack a critical factor. A depth map is obtained by identifying the frame where it appears most in focus. These approaches typically require densely sampled stacks containing 50–100 images. For example, methods \cite{subbarao1995accurate, muhammad2012sampling} used real and synthetic cone stacks with 90 and 97 images, respectively, while \cite{mahmood2012nonlinear, moeller2015variational} relied on Sine and Wave stacks with 60 images each, and \cite{suwajanakorn2015depth} used up to 33 images per stack. Advances in deep learning have greatly improved SFF by enabling continuous depth estimation rather than discrete outputs and thus reducing stack size. However, deep SFF models still depend on large focal stacks for training and evaluation. Although modern datasets use fewer images, typical stacks still contain around 5–15 frames. For instance, DDFFNet \cite{hazirbas2019deep}, one of the earliest deep SFF networks, was trained on a large dataset named DDFF of over 5,000 focal stacks, each containing 10 images captured using a light-field camera. Similarly, the Middlebury \cite{scharstein2014high} and FlyingThings3D (FT) \cite{mayer2016large} datasets, repurposed for SFF in \cite{wang2021bridging}, both consist of 15 images per stack, while the Focus-on-Defocus (FoD) dataset \cite{maximov2020focus} provides 500 stacks with 5 images each, representing the smallest so far. In contrast, a recent microscopic dataset \cite{dogan2025swin} employed sequences of 180 high-resolution images. Despite these reductions, deep SFF models still require more inputs than stereo or monocular methods, making it crucial to minimize the number of focal images for faster inference and real-time deployment.

Mostly, in deep learning SFF approaches, the focal stack is fed directly into a deep model to extract features (focus volumes), without using meaningful data augmentation beyond the raw images. However, prior studies have shown that deep encoders learn more effectively when supplemented with auxiliary information \cite{cheung2023survey}. Motivated by this, we hypothesize that integrating auxiliary cues with a small-sized focal stack can help the encoder learn more discriminative for accurate depth estimation.

In this paper, we propose a physics-guided augmentation of the input focal stack that significantly reduces the number of images required for deep learning based-SFF. The augmentation introduces two auxiliary cues: an AiF image estimated from the input stack, and Energy of Difference (EOD) maps, which measure the per-pixel differences between each defocused image and the AiF. These cues provide interpretable, physics-consistent depth information aligned with the principles of image formation. Building on this augmented input, we further develop a recurrent deep network that iteratively updates depth rather than collapsing the focus volume in a single step. Experiments show that the proposed augmented image stack reduces the required input to two or three images while improving depth estimates, enabling state-of-the-art performance on both existing and proposed SFF frameworks.

\section{Related Work}
\subsection{Traditional Methods}
In traditional SFF methods, depth is estimated by evaluating the focus quality of each pixel across the focal stack using a focus measure (FM), forming a three-dimensional focus volume (FV) that captures sharpness variations across slices. Common FMs are categorized as derivative-based, transform-based, or statistics-based. Derivative-based methods \cite{nayar1994shape,jeon2019ring} rely on image gradients or second-order derivatives such as Sobel and Laplacian operators, while transform-based methods \cite{ali20203d, nie2021focus} assess high-frequency energy in domains like DCT or DWT. Statistics-based approaches, including gray level variance \cite{krotkov1988focusing} and higher order moments \cite{zhang2000new}, measure local intensity variations to estimate focus. More recently, a dual-stage method \cite{ashfaq2026dual} addressed the limitations of direct gray scale conversion by transforming color focal stacks into an informative scalar representation before applying focus measures for depth estimation. Another recent method \cite{ashfaq2026depth} generates multiple FVs from directional focus responses and fuses them into a single scalar response using a vector-to-scalar fusion approach. Since FVs are often noisy, they are refined through filtering or optimization. Linear filters are simple but sensitive to window size \cite{nayar1994shape}, whereas nonlinear diffusion \cite{mahmood2012nonlinear} better preserves edges. Optimization-based methods \cite{ali2021robust,ali2022energy} further improve accuracy. Finally, depth is estimated using a Winner-Takes-All strategy and refined through post-processing such as median filtering \cite{minhas2009}.

\subsection{Deep Methods}
In deep learning-based SFF, convolutional neural networks predict depth from a focal stack through two main stages: constructing a deep focus volume and regressing depth from it. In the first stage, encoder-decoder (ED) architectures extract deep focus measures, using either 2D or 3D convolutions. 2D ED networks process each slice independently and concatenate features to form the focus volume, as seen in \cite{hazirbas2019deep, lu2021self}, while 3D ED models, such as \cite{wang2021bridging}, directly apply 3D convolutions to capture spatial and focus-level correlations simultaneously. The second stage estimates depth from the focus volume, often replacing the non-differentiable $\arg\max$ with a differentiable soft $\arg\max$ \cite{kendall2017end}, where probabilities computed via softmax yield smooth depth regression. Multi-scale aggregation and uncertainty-aware formulations further improve prediction accuracy \cite{yang2022deep, won2022learning}. Recent studies also integrate priors from monocular depth estimation \cite{ganj2025hybriddepth} and employ transformer-based architectures such as Swin-Transformer \cite{dogan2025swin} for enhanced global context and frequency-aware focus representation. Additionally, a recent study \cite{ashfaq2026robust} incorporated multiple traditionally computed multi-scale focus volumes into a deep network for iterative depth extraction and refinement. Overall, deep methods \cite{yang2023aberration, fujimura2024deep} overcome many limitations of traditional techniques by enabling end-to-end learning and improved robustness to noise.

\section{Motivation}

\begin{figure}[hbtp]
    \centering
    \includegraphics[width=3.3in]{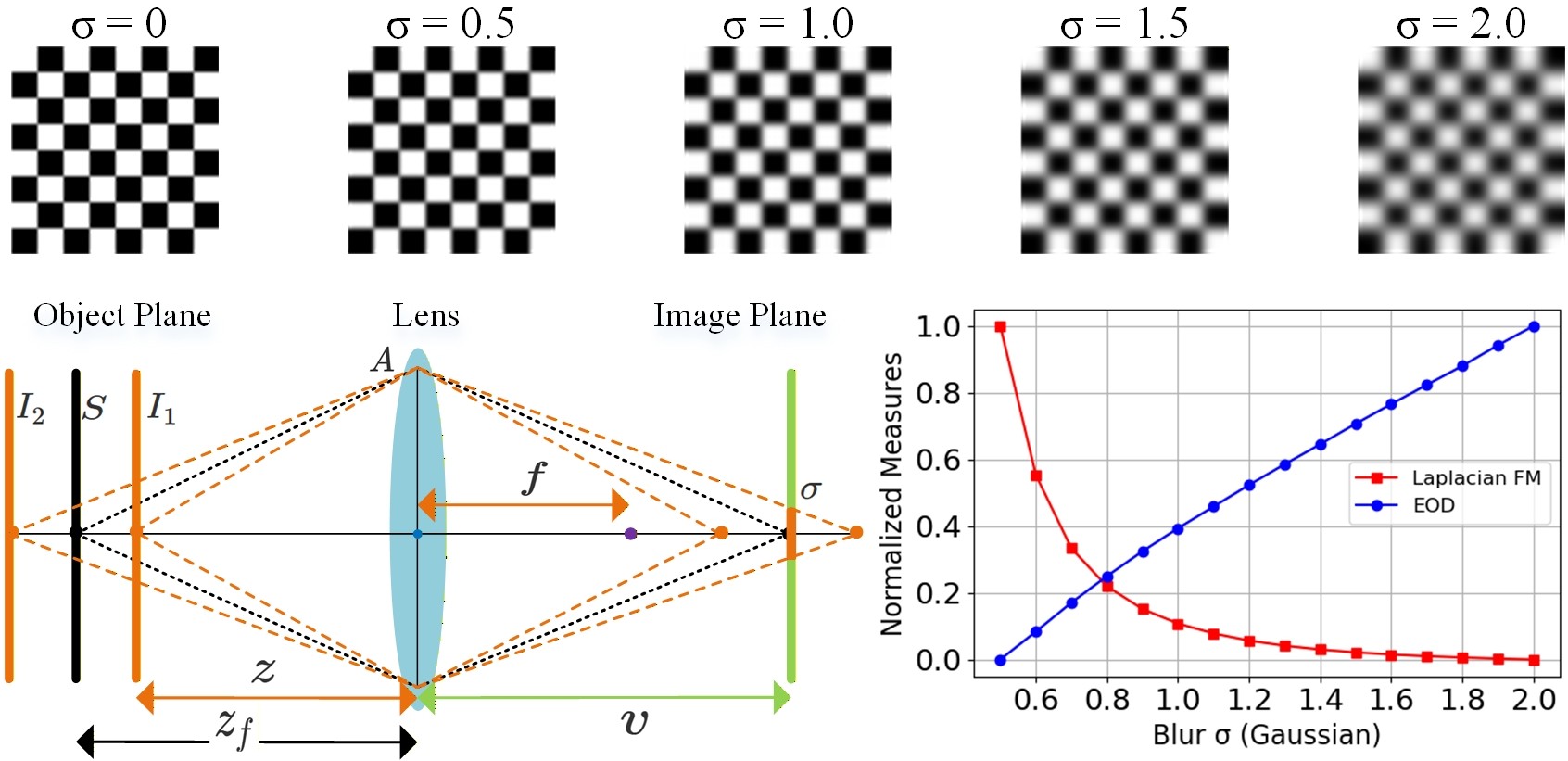}
    \caption{First Row: A synthetic AiF image of size $64 \times 64$ is generated and defocused images by convolving with a Gaussian with $\sigma \in\left\{0.5,1.0,1.5,2.0 \right\}$. Second Row: (left) Illustration of defocused image formation, (right) behavior of EOD and Laplacian focus measure (FM).}
    \label{sigmaEOD}
\end{figure}

\begin{figure*}[t]
    \centering
    \includegraphics[width=6.5in]{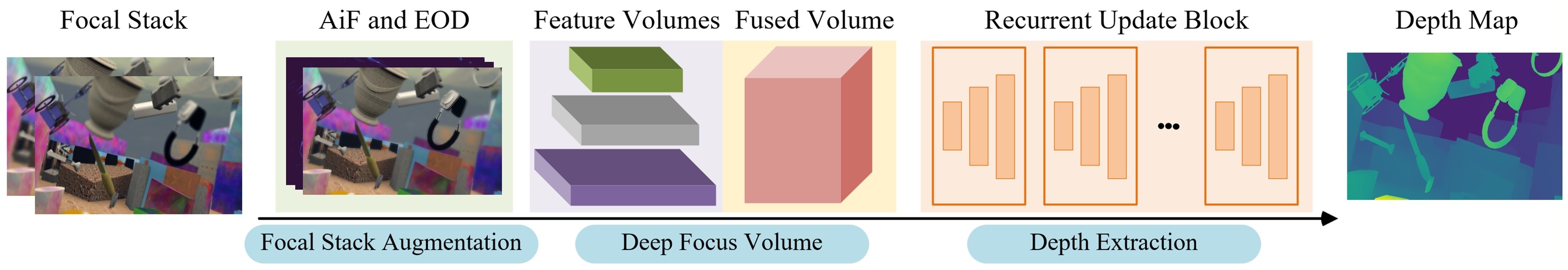}
    \caption{The proposed framework has three modules: Focal Stack Augmentation, Deep Focus Volume, and Depth Extraction.}
    \label{mainMethod}
\end{figure*}
When a point on the object plane is out of focus, the corresponding light rays converge either in front of or behind the image plane forming a circle of confusion (CoC) with radius $\sigma$, as shown in the second row (left)\Cref{sigmaEOD}. The value of $\sigma$ quantifies the amount of defocus: smaller values indicate sharp, in-focus regions, while larger values correspond to stronger blur \cite{gur2019single, wijayasingha2024camera}.  According to the thin-lens law, the relationship between $\sigma$, scene depth, and camera parameters is given by \cite{si2023fully};
\begin{equation}\label{eq:sigmaDepth}
\sigma(p)=  \left|\frac{ z(p) - z_f}{z(p)} \right| \cdot  C
\end{equation}
where $z(p)$ denotes the scene depth at pixel $(p)$, $z_f$ is the focus distance,  and $C$ is the term containing camera parameters.  An AiF image $S(p)$ and a defocused image $I(p)$ are related as;
\begin{equation}\label{eq_img_from}
I(p) = S(p) \circledast h(p;\sigma(p)),
\end{equation}
where $\circledast$ is the convolution and $h(p;\sigma(p))$ is Point Spread Function (PSF), typically approximated by a Gaussian function. We define the Energy of Difference (EOD) map as;
\begin{equation}
\label{eq_EOD}
E(p)=|I(p)-S(p)|^2.
\end{equation}
\begin{proposition}: The energy of difference $E(p)$  is directly proportional to the blur level $\sigma$ and inversely proportional to the focus measure (FM) operator.
\end{proposition}

\begin{proof} 
Taking the 2D Fourier transform of $S(p)$,  $I(p)$, and $h(p; \sigma(p)$ and from \Cref{eq_img_from} can be written as;
\begin{align*}
\mathcal{I}(u,v) = \mathcal{S}(u,v)\mathcal{H}(u,v).
 \end{align*}
Computing the difference between $\mathcal{I}(u,v)$ and $\mathcal{S}(u,v)$,
\begin{align*}
\mathcal{D}(u,v) =  \mathcal{I}(u,v)- \mathcal{S}(u,v) = \mathcal{S}(u,v) \big(1 - \mathcal{H}(u,v)\big).
 \end{align*}
By substituting  $\mathcal{H}(u,v)$ with  $w=(u^2+v^2)$;
\begin{align*}
\mathcal{E}(u,v;\sigma) &=  \iint |\mathcal{S}(u,v)|^2 \, \big|1 - \mathcal{H}(u,v)\big|^2 \, dudv,\\
&= \iint |\mathcal{S}(u,v)|^2 \left(1 - e^{(-2\pi^2 \sigma^2 w)}\right)^2 \, dudv,
\end{align*}
Hence, $\mathcal{E}(u,v;\sigma)$ increases with $\sigma$;  when $\sigma = 0$ no blur is there and $\mathcal{E}(u,v;0)=0$.
\end{proof}

The relationship $E(p) \propto \sigma(p)$ and $FM(p) \propto 1/\sigma(p)$, hence $E(p) \propto 1/FM(p)$, is illustrated in the second row (right) of Fig.~\ref{sigmaEOD}, where a synthetic AiF image is progressively blurred with Gaussian kernels. As shown, EOD retains sensitivity over a wider range of focal distances (larger $\sigma$), while FM drops rapidly with small increases in $\sigma$, limiting its effective range. Consequently, conventional SFF methods require larger focal stacks to capture portions of $\sigma$'s or depth variations. In contrast, EOD can encode a broader $\sigma$ range, making it an effective auxiliary cue for deep models in depth estimation.

The limited depth of field (DoF) of large-aperture lenses prevents direct acquisition of AiF images from imaging systems. They are often estimated from focal stacks in SFF methods \cite{jiang2024learning} as byproducts or used as internal supervisory signals in monocular depth estimation \cite{gur2019single, wijayasingha2024camera}. Providing an estimated AiF image as auxiliary cue, together with defocus, images as input to a deep model can further enhance accurate depth learning. Therefore, we propose a compact focal stack consisting of two or three images, combined with EOD maps and AiF images, for SFF frameworks.

\section{Proposed Framework}

\begin{figure*}[h]
    \centering
    \includegraphics[width=6.9 in]{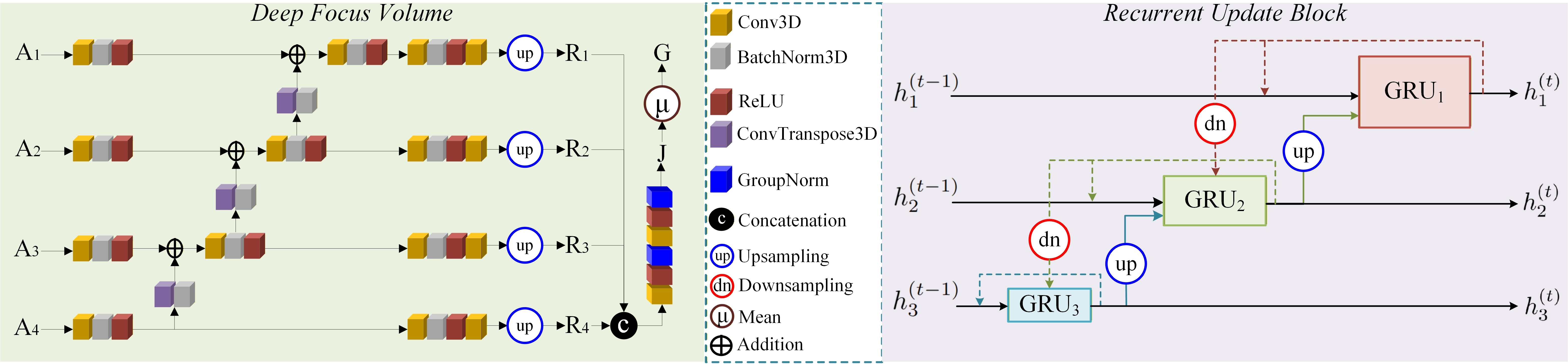}
    \caption{\textbf{Left}: The architecture of the DFV module. For brevity, the feature encoder from which the feature volumes 
    ${\{\mathbf{A}_1, \mathbf{A}_2, \mathbf{A}_3, \mathbf{A}_4\}}$ are extracted is not shown. 
    \textbf{Right}: The recurrent update block, where GRUs operating at coarse, 
    medium, and fine scales exchange multi-scale information among themselves.}
    \label{DFV_GRU_fig}
\end{figure*}

The proposed framework consists of three modules: Focal Stack Augmentation, Deep Focus Volume, and Depth Extraction as depicted in \Cref{mainMethod}.

\subsection{Focal Stack Augmentation}

Given two or more defocused images $I_{z,c}(p)$ where $z\in\{1,2,3\}$ and $c\in\{R,G,B\}$, we first estimate an AiF image using a simple weighted sum method\cite{huang2007evaluation}. For each slice, we compute directional Laplacian responses along four orientations, then aggregate their magnitudes to obtain a per-pixel focus measure as
\begin{equation}
   F_{z,c}(p) = \sum_\theta| I_{z,c}(p) \circledast K_{\theta}|, 
\end{equation}
where $\circledast$ is the convolution operation, $K_{\theta}$ are Laplacian kernel in the directions $\theta \in \{0^{\circ}, 45^{\circ}, 90^{\circ}, 135^{\circ}\}$. The Aif $\hat{S}_{c}(p)$ and EOD maps are computed as;
\begin{align}
\hat{S}_{c}(p) &=\; \sum_{z} w_{z,c}(p)\cdot I_{z,c}(p),\\
E_{z,c}(p) &=|I_{z,c}(p) - \hat{S}_c(p)|^2. 
\end{align}
where $w_{z,c}(x,y)$ weights are computed by applying softmax on the focus measures $F_{z,c}(p)$. The augmented  input stack is obtained by concatenating the auxiliary cues  $E_{z,c}(p)$, $\hat{S}_{c}(p)$ with  images;
\begin{equation}
    \bar{I}_{z,c}(p) = \text{cat}\big(I_{z,c}(p),\hat{S}_c(p), E_{z,c}(p) \big).
\end{equation}

\subsection{Deep Focus Volume}
The augmented stack $\bar{\mathbf{I}}\in\mathbb{R}^{Z\times C\times H\times W}$ is first processed by a ResNet18-based encoder~\cite{he2016deep} pretrained on ImageNet-1K~\cite{deng2009imagenet}, which generates four dense feature volumes $\{\mathbf{A}_1, \mathbf{A}_2, \mathbf{A}_3, \mathbf{A}_4\}$ at different resolutions. Each $\mathbf{A}_n$ is refined through a residual decoder based on 3D convolutional operations that capture spatial and slice-wise context. The resulting decoded features $\{\mathbf{R}_1, \mathbf{R}_2, \mathbf{R}_3, \mathbf{R}_4\}$, where each $\mathbf{R}_n \in \mathbb{R}^{Z\times1\times H\times W}$, are upsampled to the input resolution for spatial alignment and also for extracting initial depth for GRU based refinement. To produce a fused representation independent of the number of slices $Q$, the decoded volumes are concatenated along the slice dimension and projected into a fixed-dimensional feature space using 3D convolutions to capture inter-slice dependencies, yielding $\mathbf{J}\in\mathbb{R}^{Z\times C\times H\times W}$. The final fused feature map $\mathbf{G} \in \mathbb{R}^{C\times H\times W}$ is obtained by averaging $\mathbf{J}$ across slices, resulting in a consistent feature representation with channel dimension $C=64$. Since $\mathbf{G}$ maintains a fixed channel dimension, the model can be trained and inferred on input stacks with arbitrary numbers of slices.

\subsection{Depth Extraction}

The final module progressively refines an initial estimate through recurrent updates. The process starts from an initial depth map $D_0$, derived from the fused volume $\mathbf{G}$ using a soft-argmax operation~\cite{kendall2017end}. The refinement is carried out by a hierarchy of $K=3$ convolutional GRU (ConvGRU) units, each operating at a distinct spatial scale. Let ${D}_t$ denote the predicted depth at iteration $t$, and $h_k^{(t)} \in \mathbb{R}^{C_h\times H_k\times W_k}$ the hidden state of the $k$-th GRU layer ($C_h=128$). The hidden state is updated as
\begin{equation}
h_k^{(t)} = \mathrm{GRU_k}\big(h_k^{(t-1)}, \mathbf{x}_k^{(t)}\big),
\label{eq:gru_update}
\end{equation}
where $\mathbf{x}_{k}^{(t)}$ is the input to $\mathrm{GRU}_k$. Each GRU follows the standard gating mechanism with convolutional operations for the update, reset, and candidate states. The refinement proceeds from the coarsest $(k=3)$ to the finest $(k=1)$ scale, with hidden states propagated upward. Intermediate levels blend features interpolated from coarser scales with locally pooled fine-scale information, as illustrated in \Cref{DFV_GRU_fig}. The finest GRU receives additional fused focus–depth features $\mathbf{B}^{(t)}$, obtained by combining the fused focus volume $\mathbf{G}$ with the previous depth estimate ${D}_{t-1}$ using a lightweight convolutional block. Thus, $\mathrm{GRU}_1$ refines depth using both multi-scale context and fused focus–depth cues. At each iteration, the hidden state $h_1^{(t)}$ produces a depth residual $\Delta D_t$ through two successive $3\times3$ convolutions followed by ReLU activation. The updated depth is
\begin{equation}
{D}_{t} = {D}_{t-1} + \Delta {D}_t,
\label{eq:depth_update}
\end{equation}
which progressively refines the estimate over iterations. Since refinement operates at reduced resolution, a learnable upsampling~\cite{teed2020raft,lipson2021raft} reconstructs the full-resolution depth map $\hat{D}_t$. After $T$ iterations, the final high-resolution prediction $\hat{D}_T$ is obtained.

\subsection{Loss}
During training, each intermediate depth estimate $\hat{D}_t$ is supervised with the ground truth depth map ${D}'$ using a mean squared error loss with progressively increasing weights across iterations as:
\begin{equation}
    \mathcal{L} = \sum_{t=1}^{T} \alpha^{T-t}\,\big({D}' - \hat{D}_t\bigr)^2,
\end{equation}
where we set $\alpha = 0.9$. At inference, we simply take the final iteration result $\hat{D}_T$.

\section{Results and Discussion}

\subsection{Experimental Setup}

Our model has a total of 8.7M trainable parameters. The implementation is done in PyTorch \cite{paszke2019pytorch}. The model is optimized using the AdamW optimizer \cite{loshchilov2017decoupled}. The maximum learning rate is set to $1\times10^{-4}$ and is scheduled with a one-cycle learning rate policy. The model employs a recurrent refinement process with 4 iterative updates, where the inputs are downsampled by a factor of 2 before being passed into the GRU layers. It contains 3 GRU layers in total, with the first operating at half of the original resolution, the second at two times lower resolution relative to the first, and the third at four times lower resolution relative to the first. Each GRU layer is defined with a hidden dimension of 128. For evaluation and ablation, we use four diverse datasets: two synthetic, FT \cite{mayer2016large}, and FoD \cite{maximov2020focus}, and two real-world, DDFF \cite{hazirbas2019deep} and Mobile Depth \cite{suwajanakorn2015depth}. Dataset-specific training configurations are provided in \Cref{dataset_config}, while the Mobile dataset is used exclusively for zero-shot generalization.

\begin{table}[htbp]
  \centering
  \caption{Training settings for each dataset, including focal stack size, number of input images used, slice indices, and patch size.}
    \begin{tabular}{ccccc}
     \toprule
    Dataset & Size & Input & Indices & Patch \\
    \hline
    FT    & 15    & 2 ; 3 & [5,11] ; [4,8,12] & 512x512 \\
    \hline
    FOD   & 5     & 2 ; 3 & [2,4] ; [1,3,5] & 256x256 \\
    \hline
    DDFF  & 10    & 2 ; 3 & [4,7] ; [3,6,8] & 224x224 \\
    \bottomrule
    \end{tabular}%
  \label{dataset_config}%
\end{table}%

\subsection{EOD Analysis}
\begin{figure}[hbtp]
    \centering
    \includegraphics[width=3.3in]{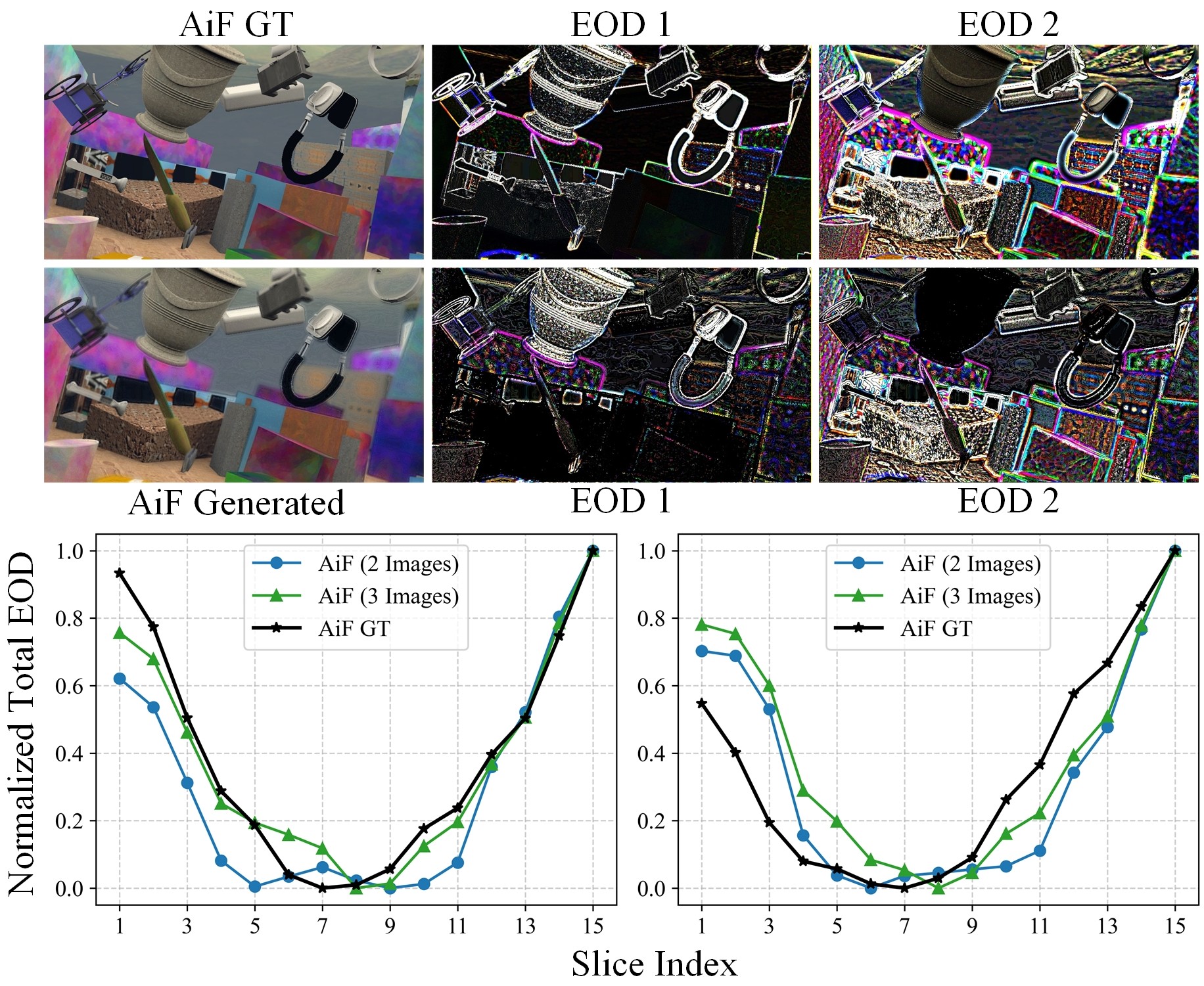}
    \caption{EOD visualization: (Row:1) GT AiF and corresponding EOD maps, (Row:2) Estimated AiF and corresponding EOD maps, (Row:3) EOD computed across the focal stack for AiFs estimated using 2 and 3 slices compared with the GT AiF.}

    \label{EOD_Fig}
\end{figure}

To evaluate the quality of the synthesized AiF image, we perform an analysis using an image pair ([5, 11]) from the FT dataset. The ground truth AiF and the AiF generated by our method, together with their corresponding EOD maps, are shown in \Cref{EOD_Fig} (rows 1–2). Despite the artifacts in the estimated AiF relative to the ground truth, the EOD map visualizations demonstrate that it effectively captures the blur variations across the input images. Figure \ref{EOD_Fig} (row 3) presents the EOD computed across the two focal stacks using the ground truth AiF and the AiFs estimated from 2 and 3 slices. As expected, the EOD is minimal at the central image of the stack and increases as we move farther from the center. Although the EODs from the estimated AiFs deviate from the ground truth, they exhibit a similar overall trend.

\begin{table}[htbp]
  \centering
  \caption{EOD ablation with different augmentation settings.}
  \resizebox{0.47\textwidth}{!}{%
    \begin{tabular}{cccccc}
    \toprule
    Ablation & MAE $\downarrow$   & RMS $\downarrow$   & SqRel $\downarrow$ & $\delta$ $\uparrow$ & $\delta^2$ $\uparrow$ \\
    \hline
    Imgs  & 6.61  & 11.82 & 9.05  & 83.22 & 87.39 \\
    AiF+Imgs & 6.22  & 11.30 & 7.81  & 84.81 & 87.74 \\
    AiF+EODs & 6.91  & 12.06 & 9.06  & 83.30 & 87.36 \\
    Imgs+EODs & 6.32  & 11.40 & 7.95  & 84.42 & 87.65 \\
    Imgs+AiF+EODs & 6.18  & 11.26 & 7.68  & 84.64 & 87.72 \\
    \bottomrule
    \end{tabular}%
    }
  \label{EOD_Ablation_Table}%
\end{table}%

To evaluate the effect of the proposed augmentations and understand the contribution of AiF, we conducted an ablation study. For a clearer assessment of AiF, we used the GT AiF during evaluation. As shown in \Cref{EOD_Ablation_Table}, using only the input images serves as the baseline. Adding the GT AiF improves all metrics, showing its benefit in preserving focus continuity. Using AiF together with only the EODs leads to a slight drop due to limited spatial cues, while combining EODs with the input images recovers the accuracy. The best performance is achieved when all three components (Imgs + AiF + EODs) are used jointly. The qualitative results in \Cref{EOD_Ab_Fig} support this pattern. Using only the images leads to loss of scene detail, while combining all components produces clearer structure and lower error values.

\begin{figure}[h]
\centering
\includegraphics[width=3.3in]{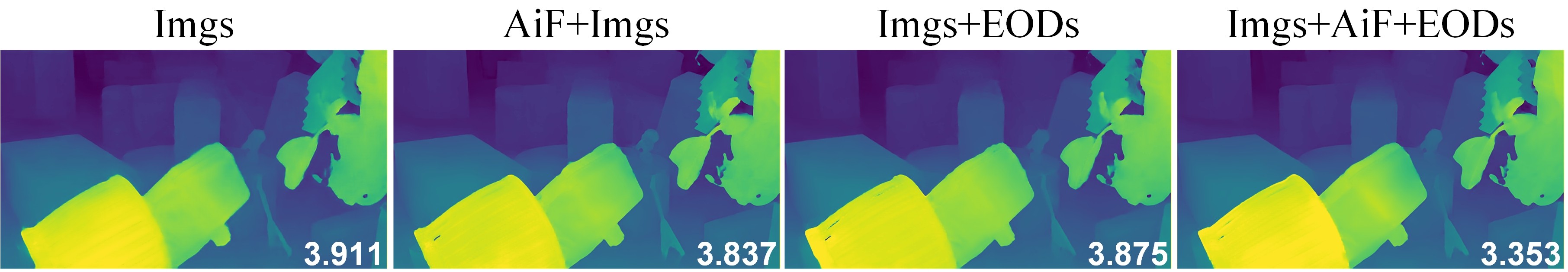}
\caption{Qualitative Results for EOD Ablation. The numbers on each map denote RMS error.}
\label{EOD_Ab_Fig}
\end{figure}

\begin{table*}[htbp]
  \centering
 \caption{Quantitative comparison on FT and FOD datasets. The upper part shows baseline results using full focal stacks, while the lower part shows performance with the proposed EOD augmentation.}
  \resizebox{1\textwidth}{!}{%
    \begin{tabular}{cccccccc|ccccccc}
    \toprule
          & \multicolumn{7}{c|}{FT}                                & \multicolumn{7}{c}{FOD} \\
          \cline{2-8}\cline{9-15}
    Method      & Input  & MAE $\downarrow$   & RMS $\downarrow$   & AbsRel $\downarrow$ & $\delta$ $\uparrow$ & $\delta^2$ $\uparrow$ & $\delta^3$ $\uparrow$ & Input  & MAE $\downarrow$   & RMS $\downarrow$   & AbsRel $\downarrow$ & $\delta$ $\uparrow$ & $\delta^2$ $\uparrow$ & $\delta^3$ $\uparrow$ \\
          \hline
          & \multicolumn{14}{c}{Using Full Focal Stack} \\
          \hline
    AiFDNet & 15    & 6.812 & 13.14 & 0.731 & 85.43 & 87.67 & 88.87 & 5     & \textbf{0.071} & 0.156 & \textbf{0.118} & \textbf{84.75} & \textbf{94.67} & 97.46 \\
    DFV-FV & 15    & 6.332 & 12.09 & 0.899 & 85.09 & 87.60 & 89.52 & 5     & 0.078 & 0.160 & 0.142 & 80.93 & 94.19 & \textbf{97.57} \\
    DFV-Diff & 15    & 5.509 & 10.65 & 0.615 & 86.18 & 88.09 & 89.93 & 5     & 0.077 & 0.165 & 0.133 & 81.77 & 93.98 & 97.43 \\
    DWild & 15    & 5.542 & 10.44 & 0.611 & 86.35 & 88.21 & 89.84 & 5     & 0.073  & \textbf{0.154}  & 0.138  & 83.12  & 94.45 & 97.26 \\
    
    \hline
          & \multicolumn{14}{c}{Using Proposed Augmented Stack} \\
          \hline
  
  AiFDNet   & 3     & 6.040 & 11.35 & 0.654 & 85.18 & 87.82 & 89.72 & 3     & 0.084 & 0.167 & 0.164 & 77.64 & 91.89 & 96.13 \\

   DFV-FV   & 3     & 5.857 & 10.95 & 0.639 & 85.54 & 87.97 & 89.85 & 3     & 0.094 & 0.176 & 0.189 & 72.82 & 91.38 & 96.47 \\
 
  DFV-Diff  & 3     & 5.909 & 11.04 & 0.640 & 85.54 & 87.95 & 89.82 & 3     & 0.086 & 0.169 & 0.163 & 77.78 & 92.72 & 96.68 \\
  Ours     & 3     & \textbf{1.776} & \textbf{4.037} & \textbf{0.059} & \textbf{97.95} & \textbf{99.26} & \textbf{99.55} & 3     & 0.079 & 0.154 & 0.157 & 79.15 & 93.94 & 97.36 \\
 \hline
   AiFDNet   & 2     & 6.227 & 11.57 & 0.669 & 84.96 & 87.65 & 89.66 & 2     & 0.110 & 0.196 & 0.235 & 67.01 & 87.99 & 94.98 \\
    DFV-FV   & 2     & 6.104 & 11.28 & 0.658 & 85.15 & 87.84 & 89.77 & 2     & 0.110 & 0.195 & 0.225 & 63.01 & 87.57 & 95.54 \\
     DFV-Diff  & 2     & 6.076 & 11.35 & 0.656 & 85.23 & 87.84 & 89.75 & 2     & 0.113 & 0.195 & 0.240 & 61.76 & 87.12 & 95.59 \\
     Ours     & 2     & 2.112 & 4.651 & 0.075 & 97.21 & 99.00 & 99.37 & 2     & 0.102 & 0.183 & 0.196 & 68.76 & 90.41 & 96.46 \\
   
          \bottomrule
    \end{tabular}%
     }
  \label{EOD_effect_table}%
\end{table*}%

\begin{table}[htbp]
  \centering
  \caption{Results on DDFF Validation set using the proposed augmentation}
  \resizebox{0.45\textwidth}{!}{%
    \begin{tabular}{ccrrrr}
    \toprule
      Method    & Input & \multicolumn{1}{c}{MAE $\downarrow$} & \multicolumn{1}{c}{RMS $\downarrow$} & \multicolumn{1}{c}{$\delta$ $\uparrow$} & \multicolumn{1}{c}{$\delta^2$ $\uparrow$} \\
         \hline
   
    AiFDNet     & 3     & \textbf{0.0028} & 0.0084 & 98.04 & 98.63 \\
    DFV-FV      & 3     & 0.0035 & 0.0072 & 97.75 & 99.01 \\
    DFV-Diff    & 3     & 0.0040 & 0.0076 & 97.68 & 99.03 \\
     Ours       & 3     & 0.0031 & \textbf{0.0065} & \textbf{98.60} & \textbf{99.39} \\
   \hline 
    AiFDNet     & 2     & 0.0028 & 0.0082 & 98.12 & 98.71 \\
    DFV-FV      & 2     & 0.0036 & 0.0074 & 97.64 & 98.93 \\ 
    DFV-Diff    & 2     & 0.0036 & 0.0073 & 97.64 & 98.98 \\
    Ours        & 2     & 0.0031 & 0.0065 & 98.52 & 99.36 \\
   
           \bottomrule
    \end{tabular}%
    }
  \label{DDFF_EOD_table}%
\end{table}%

\begin{figure}[t]
    \centering
    \includegraphics[width=3.3in]{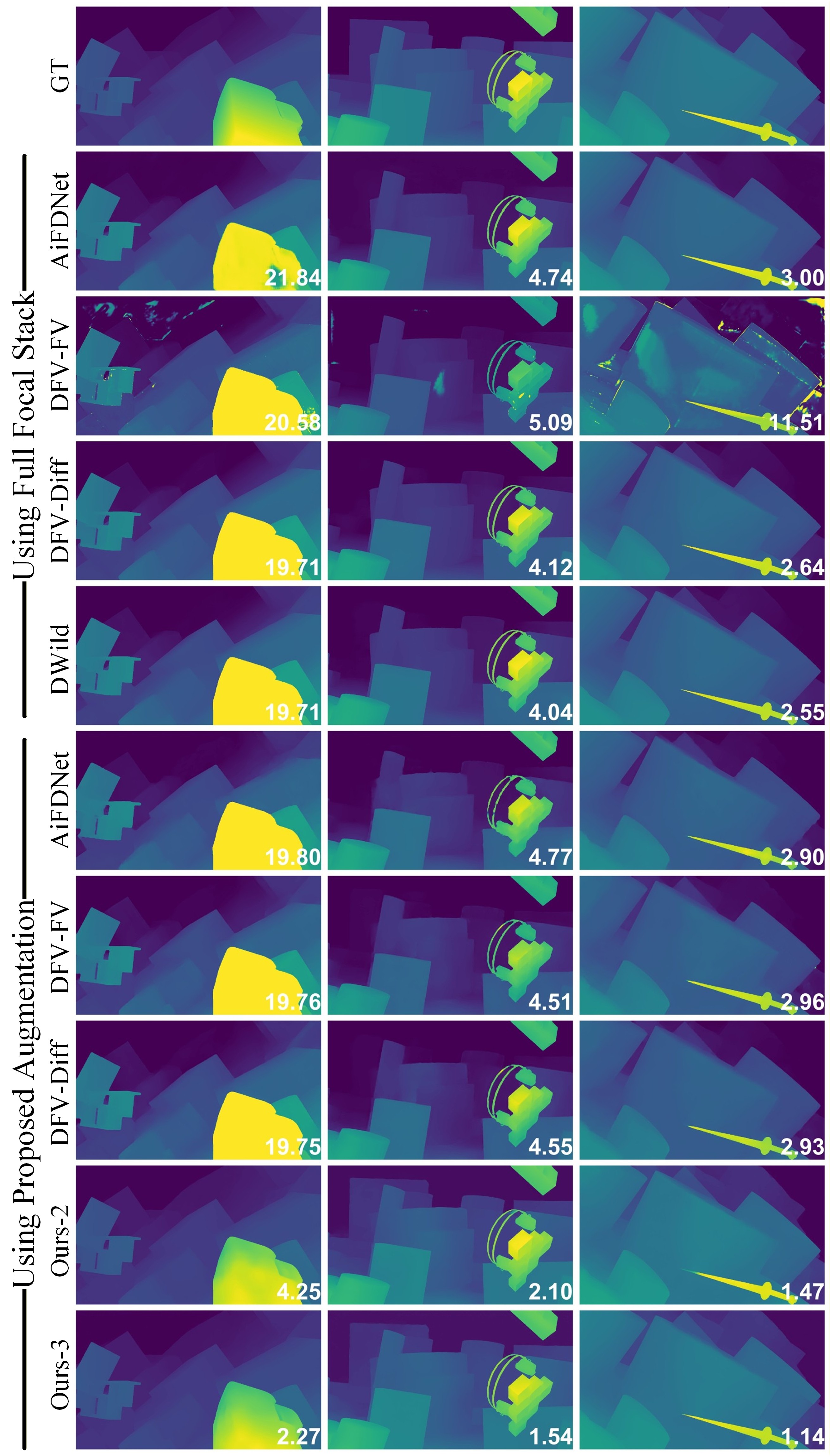}
    \caption{Qualitative results on the FT dataset. Comparison between baseline models trained with full focal stacks and models using the proposed EOD augmentation. Numbers on each map denote the corresponding RMS error with respect to the GT.}
    \label{FT_qual}
\end{figure}

\begin{figure*}[hbtp]
    \centering
    \includegraphics[width=6.9in]{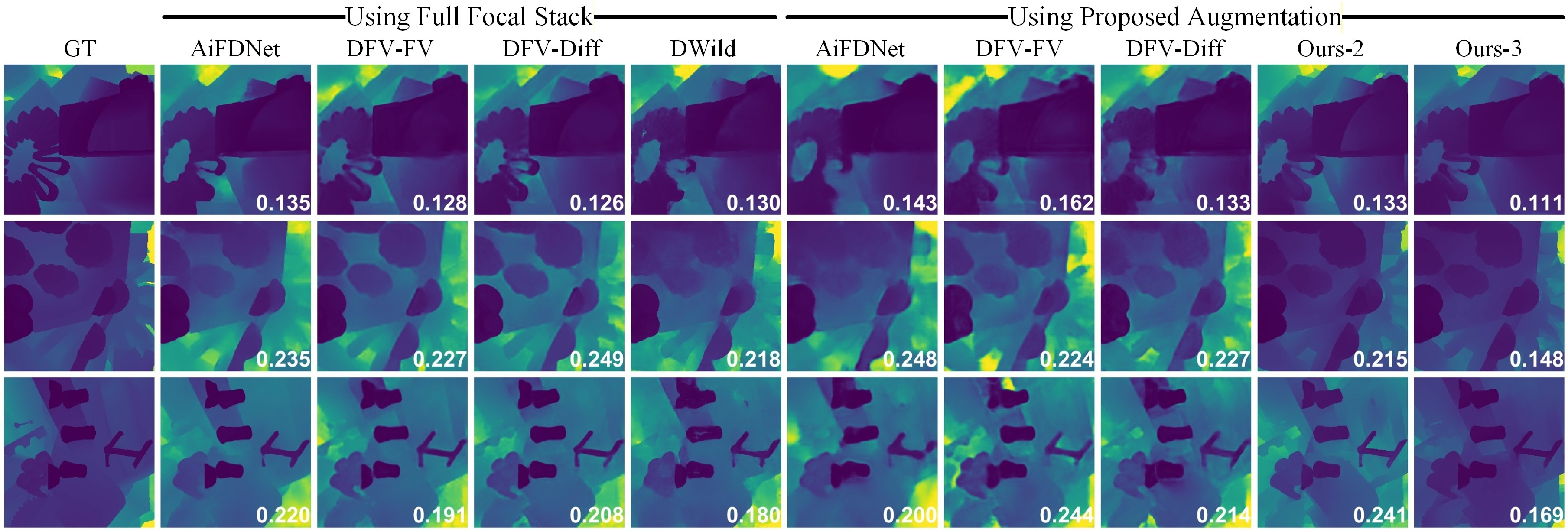}
    \caption{Qualitative results on the FOD dataset. Comparison between baseline models trained with full focal stacks and models using the proposed EOD augmentation. Numbers on each map denote the corresponding RMS error with respect to the GT. Our method produces sharper boundaries; zoom in for details.}
    \label{FOD_qual}
\end{figure*}

\begin{figure}[h]
    \centering
    \includegraphics[width=3.3in]{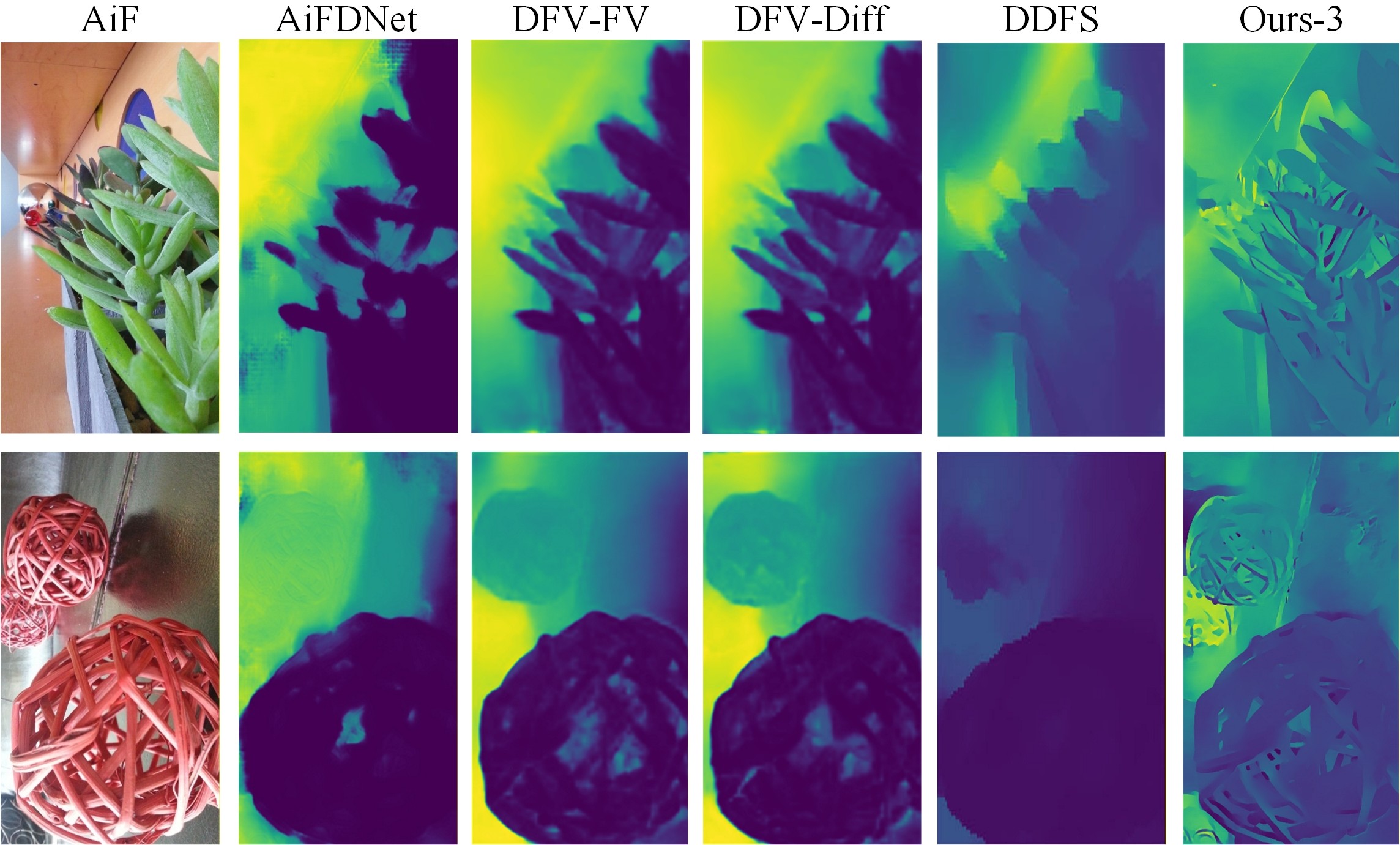}
    \caption{Qualitative comparison on the Mobile dataset. Ours is only from 3 input images.}
    \label{Mobile_qual}
\end{figure}

\subsection{Comparative Analysis}

To evaluate the effectiveness of the proposed EOD augmentation and the model, we integrated the augmentation step into several existing deep SFF frameworks, namely AiFDNet~\cite{wang2021bridging}, DFV-FV, and DFV-Diff~\cite{yang2022deep}, and compared their performance. Each method was retrained using only 2 and 3 input images on the FT and FOD datasets. The quantitative results for both datasets are summarized in \Cref{EOD_effect_table}. For reference, the table also includes the performance of these methods when trained with their full focal stacks (15 images for FT and 5 for FOD), allowing direct comparison of how performance changes when the number of input images is reduced. In addition, we include another state-of-the-art baseline, DWild~\cite{won2022learning}, for further comparison. The lower portion of \Cref{EOD_effect_table} reports results obtained after incorporating the proposed EOD augmentation, while the upper portion lists the original baseline performances using full focal stacks.

As observed on the FT dataset, the performance of AiFDNet and DFV-FV not only remained comparable but even improved when trained with the proposed augmentation using only two or three input images. For DFV-Diff, the performance stayed almost on par with the original across all metrics. The proposed model, however, outperformed all existing methods by a clear margin. A similar trend appears on the FOD dataset. When the augmentation was applied to existing methods, their performance remained highly competitive even with far fewer input images. Moreover, the three-image version of our model achieves lower RMS than most five-image competitors, highlighting the efficiency of the EOD formulation. After rounding, both our method and DWild appear to obtain the lowest RMS. However, using full precision values, DWild is slightly better, which is why it is bolded in the table. It is also worth noting that for the FOD dataset, some competing full stack methods provide official checkpoints trained with additional external data, giving them a potential advantage. Despite this, our model, trained with less data and fewer input images, still achieves competitive performance.

We also evaluate on the DDFF dataset using its evaluation split, as shown in \Cref{DDFF_EOD_table}. Since the official benchmark is no longer active, we train all models on the DDFF training set with the proposed augmentation and report validation results. The results show that with the proposed augmentation, existing methods perform well even with only 2 or 3 input images. Furthermore, our proposed model surpasses all comparative methods in most metrics, with its 2-image variant already achieving a lower RMS error than every baseline, including their 3-image versions.

The qualitative results for the FT dataset are shown in \Cref{FT_qual}. The figure compares baseline models using full focal stacks with their EOD-augmented versions trained on only two input images. As observed, the visual quality remains nearly identical, with AiFDNet and DFV-FV even showing clearer details in certain regions (first and third columns, respectively). Overall, these methods preserve visual consistency despite the drastic reduction from 15 to 2 images. The proposed model’s 2-image and 3-image variants are also included for comparison. These variants produce results that are visually closest to the GT, which is also confirmed by the RMS values shown on each map, as they are the lowest among all compared methods.

Similarly, \Cref{FOD_qual} presents the qualitative results for the FOD dataset. The figure includes the full stack results of the baseline models, the 3 image variants of AiFDNet, DFV-FV, and DFV-Diff with the proposed augmentation, as well as the proposed 2-image and 3-image variants for comparison. Although a slight decline in performance is observed in some cases, likely due to the limited focus cues available in the selected images, the overall results remain visually comparable. While the proposed model shows some contrast variation with respect to the GT depth, it still preserves sharper object boundaries and clearer edges, demonstrating the effectiveness of the iterative refinement process.

Lastly, we evaluate zero-shot generalization on the Mobile dataset~\cite{suwajanakorn2015depth}. A general model was trained sequentially, first on FT, then fine-tuned on FOD, and finally on a small 20-stack synthetic HCI dataset~\cite{honauer2016dataset}. For comparison, we also include the recent DDFS method~\cite{fujimura2024deep}. The qualitative results, shown in \Cref{Mobile_qual}, demonstrate that our approach, using only three input images while the original focal stack contains up to 33 images, achieves clear improvements. The proposed augmentation and method effectively separate background and foreground regions (as seen in the second row of results), showing strong generalization and consistent performance across both focal stacks despite the limited number of inputs.

\subsection{Ablation Study }

\begin{figure}[h]
    \centering
    \includegraphics[width=3.3in]{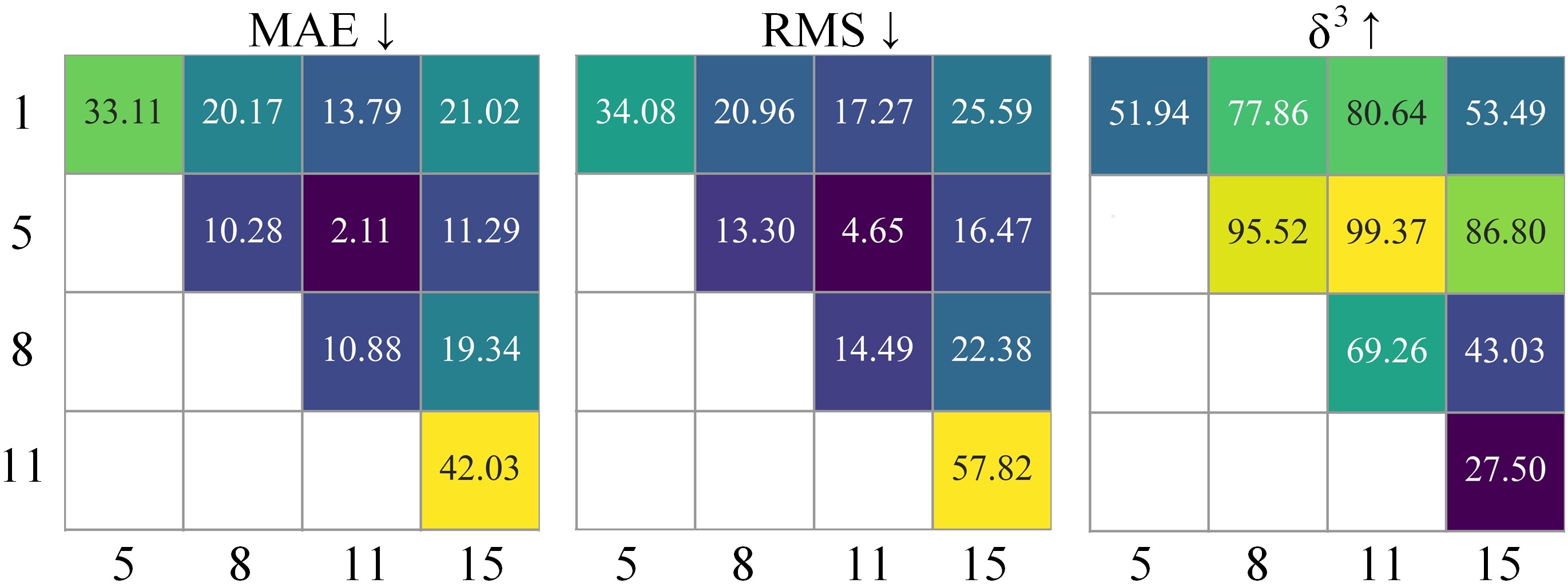}
    \caption{Pair-selection evaluation for the two-image setting. Heatmaps show different metrics for slice pairs; rows denote the first slice index and columns denote the second.}
    \label{pair_selection_fig}
\end{figure}

To evaluate the model’s robustness to PSF variations, we conduct an experiment based on different slice pair selections. The goal is to test whether the model can still extract reliable focus cues and predict accurate depth when the input images have different blur characteristics from those used during training. The model was trained on the FT dataset using the image pair [5, 11], and during testing, the input pairs were changed to other combinations. The resulting heatmaps in \Cref{pair_selection_fig} show that performance is highest when training and testing pairs match ([5, 11]), while pairs containing two similarly defocused images perform poorly due to limited focus variation. In contrast, pairs where one image focuses on the foreground and the other on the background yield notably better results.

\begin{figure}[h]
    \centering
    \includegraphics[width=3.3in]{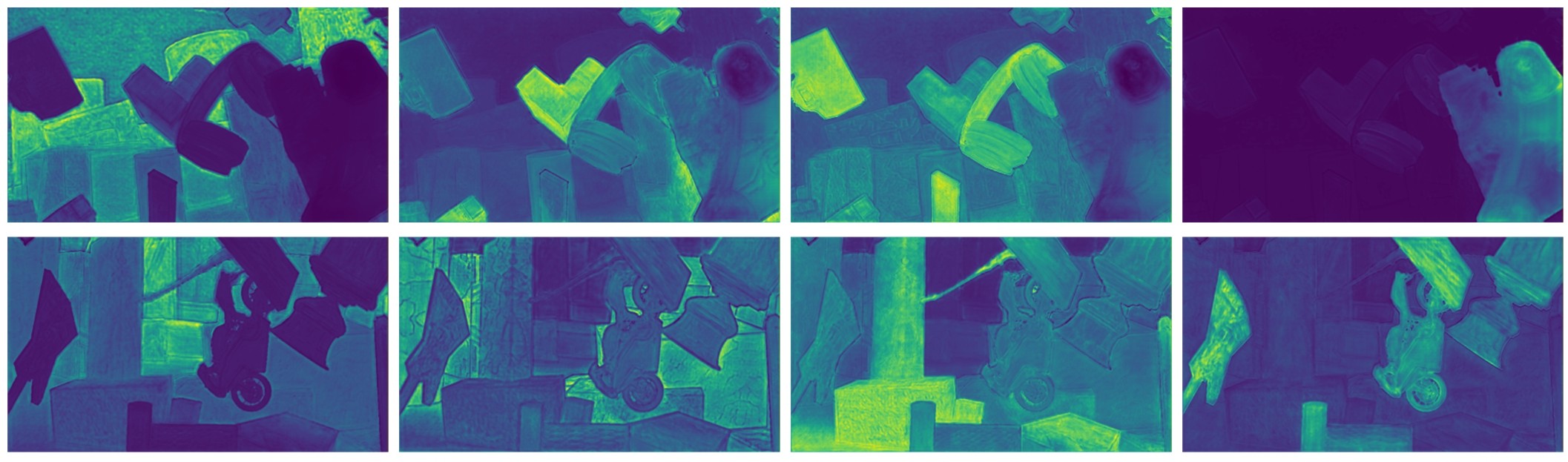}
    \caption{Activation maps across different focus regions. Despite, trained with only two defocused inputs, the model still responds strongly to all depth regions, including weakly focused areas. From left to right: background, far-mid, mid-near, and near focus areas are activated.}
    \label{activation_map_fig}
\end{figure}

To understand how the proposed model responds to regions outside the directly focused areas, we analyze the activation behavior of the fused feature volume $\mathbf{G}$. The model was trained on the FT dataset using the image pair [5, 11], where mid-far and mid-near regions were primarily in focus. We then visually inspect the spatial activation patterns within $\mathbf{G}$ to evaluate whether the model can still respond to weakly focused regions that were not well represented in the training inputs. As shown in \Cref{activation_map_fig}, even though the model was trained with only two defocused images, it exhibits strong activations across the entire scene, including areas that are weakly focused in both inputs. For reference, the four images (from left to right) show activations corresponding to background, far-mid, mid-near, and near regions.

\begin{table}[htbp]
  \centering
  \caption{Comparisons including estimated or GT AiF with images from the FT dataset.}
  \resizebox{0.45\textwidth}{!}{%
    \begin{tabular}{ccrrrrr}
    \toprule
      AiF    & Input & MAE $\downarrow$ & RMS $\downarrow$ & SqRel$\downarrow$ & $\delta\uparrow$ & $\delta^3\uparrow$ \\
         \hline
   
    Estimated AiF & 3   & 1.78     & 4.08 & 1.18 & 97.95  & 99.55    \\
    GT AiF       &  3   & 1.48     & 3.30 & 0.56 &  98.56 &  99.68   \\
    \hline 
    Estimated AiF  & 2   & 2.18    & 4.65 & 1.61 &  97.21  &  99.37  \\
    GT AiF    &    2     & 1.68    & 3.64 & 0.60  &  98.39  &  99.39  \\ 
   
    \bottomrule
    \end{tabular}%
    }
  \label{GT_AiF_table}%
\end{table}%

Finally, we conducted an experiment using both the estimated AiF and the ground-truth AiF, along with the EOD computed from 2 and 3 images of the FT dataset. The results are presented in Table \ref{GT_AiF_table}. A noticeable decrease in performance can be observed when using the estimated AiF compared to the ground-truth AiF. All evaluation metrics show improvement in both the 2-image and 3-image settings. These findings indicate that AiF plays a central role in the proposed augmentation strategy, and that enhancing the quality of the estimated AiF can lead to more accurate depth estimation. 

\subsection{Limitations and Future Work}
The proposed augmentation primarily relies on the generated AiF image, which in the current model is generally estimated from two or three input images using a simple approach. Because this estimated AiF is not always fully accurate and may contain minor artifacts, the quality of the resulting EOD maps can be slightly degraded, which in turn may negatively affect depth estimation. A promising direction for future work is to improve the AiF reconstruction process, either through a more robust and accurate estimation strategy or by learning it jointly with the deep depth estimation network as an auxiliary task.

\section{Conclusion}

In this paper, we proposed a deep model together with a focal stack augmentation framework for accurate depth estimation from a minimal focal stack. The proposed augmentation leverages an estimated all-in-focus (AiF) image and the energy of differences (EOD) between the AiF and input images to enrich the reduced stack with informative focus cues. Building on this representation, we introduced a deep network that constructs a deep focus volume from the augmented stack and iteratively refines the depth map using multi-scale ConvGRUs. Extensive experiments on diverse synthetic and real-world datasets demonstrate that the proposed model achieves state-of-the-art or highly competitive performance with only a few input images, while the proposed augmentation consistently helps existing state-of-the-art SFF models and enables comparable accuracy with significantly smaller focal stacks.

{\small
\bibliographystyle{ieee_fullname}
\bibliography{egbib}
}

\end{document}